%% file: AAAI-SunJ-6872.tex
\def\year{2020}\relax
\documentclass[letterpaper]{article} % DO NOT CHANGE THIS
\usepackage{aaai20}  % DO NOT CHANGE THIS
\usepackage{times}  % DO NOT CHANGE THIS
\usepackage{helvet} % DO NOT CHANGE THIS
\usepackage{courier}  % DO NOT CHANGE THIS
\usepackage[hyphens]{url}  % DO NOT CHANGE THIS
\usepackage{graphicx} % DO NOT CHANGE THIS
\usepackage{graphicx}  % DO NOT CHANGE THIS
\usepackage{amsmath}
\usepackage{amssymb}
\usepackage{xcolor}
\usepackage{multirow}

\newtheorem{lemma}{Lemma}

\newtheorem{proof}{Proof}
\title{New Interpretations of Normalization Methods in Deep Learning}
\author{Jiacheng Sun\textsuperscript{\rm 1}, Xiangyong Cao\textsuperscript{\rm 2}, Hanwen Liang\textsuperscript{\rm 1}, Weiran Huang\textsuperscript{\rm 1}, Zewei Chen\textsuperscript{\rm 1}, Zhenguo Li\textsuperscript{\rm 1}\\ % All authors must be in the same font size and format. Use \Large and \textbf to achieve this result when breaking a line
\textsuperscript{\rm 1}Huawei Noah Ark's Lab, \textsuperscript{\rm 2}Xi'an Jiaotong University\\
\textrm{\{sunjiacheng1, lianghanwen1, weiran.huang, chen.zewei, Li.Zhenguo\}@huawei.com}, \textrm{caoxiangyong@mail.xjtu.edu.cn} %If you have multiple authors and multiple affiliations
}
\begin{document}
\maketitle

\begin{abstract}
In recent years, a variety of normalization methods have been proposed to help train neural networks, such as batch normalization (BN), layer normalization (LN), weight normalization (WN), group normalization (GN), etc. However, mathematical tools to analyze all these normalization methods are lacking. In this paper, we first propose a lemma to define some necessary tools. Then, we use these tools to make a deep analysis on popular normalization methods and obtain the following conclusions: 1) Most of the normalization methods can be interpreted in a unified framework, namely normalizing pre-activations or weights onto a sphere; 2) Since most of the existing normalization methods are scaling invariant, we can conduct optimization on a sphere with scaling symmetry removed, which can help stabilize the training of network; 3) We prove that training with these normalization methods can make the norm of weights increase, which could cause adversarial vulnerability as it amplifies the attack. Finally, a series of experiments are conducted to verify these claims. 
\end{abstract}
%(2) Under this framework, we propose a new network pruning strategy based on centering and scaling parameters of BN, and also make the claim that the batch size of BN should be set close to the width of network; 

\section{Introduction}
Normalization has been a very effective strategy in deep learning, which speeds up training deep neural networks and also acts as a regularizer to improve generalization. This technique has been a fundamental component in many state-of-the-art algorithms, and is usually implemented by adding non-linear mappings to pre-activations or weights before activation functions. In recent years, a variety of normalization methods have been proposed, including batch normalization (BN)~\cite{ioffe2015BN}, layer normalization (LN)~\cite{lei2016LN}, instance normalization (IN)~\cite{ulyanov2016IN}, group normalization (GN)~\cite{wu2018GN}, weight normalization (WN)~\cite{salimans2016WN}, centered weight normalization (CWN)~\cite{huang2017CWN},
spectral normalization (SN)~\cite{miyato2018SN}, etc. These methods can be roughly divided into two categories: \textit{data-based normalization} and \textit{weight-based normalization}. In the following, we will first review the two categories, and then introduce the motivation and contribution of our work.

\subsection{Related work: \textit{data-based normalization} vs \textit{weight-based normalization}}
\textit{Data-based normalization} is implemented by normalizing the pre-activations/activations~\cite{mishkin2015all}, which is computed by passing the data into each layer of the network. The typical methods include BN, LN, IN and GN. More specifically, BN standardizes the pre-activation vector by transforming it into a vector with zero mean and unit variance. LN implements this standardization within one layer, which is equivalent to first centering every column of weights and then conducting the multiplicity and scaling. Therefore, it is scaling invariant with both weights and data as well as moving invariant in a certain direction. IN implements the standardization in one channel of the pre-activation, while GN balances between LN and IN and is implemented by dividing the channels into groups. 

Comparatively, the other strategy \textit{weight-based normalization} normalizes the weights directly. The typical methods are WN, CWN and 
SN. Specifically, WN~\cite{salimans2016WN} decouples the length and direction of the weight vector as BN does. CWN~\cite{huang2017CWN} standardizes the weight vector by centering at mean and scaling by its $L_2$ norm and is equivalent to weight standardization (WS)~\cite{qiao2019WS}. SN is another normalization that helps stabilize the training of discriminator in the generative adversarial network (GAN), and it is similar to WN in dividing the spectral norm of the weight matrix. This method is implemented by controlling the Lipschitz constant not too big, so it can be much more robust to noise~\cite{miyato2018SN}.

\subsection{Motivations}
Although various normalization methods have been proposed and verified to be effective in many computer vision tasks, there still exist many issues. Firstly, all the existing normalization methods are heuristic and lack of a unified theoretical understanding. Therefore, it is very important to interpret all the normalization methods from a theoretical perspective. Secondly, many normalization methods are widely used and their performance is closely related to the batch size. Take BN for instance. The small batch size makes the network training rather unstable~\cite{wu2018GN}, while a large batch size may decrease the generalization capacity~\cite{luo2018towards}. Therefore, it is a fundamental problem to choose the optimal batch size for different tasks in deep learning~\cite{masters2018revisiting,park2019effect}. Besides, the scaling parameters $\gamma$ of BN have been adopted as criteria for pruning network~\cite{liu2017slim}, which is implemented by removing units whose scaling parameter $\gamma$ is small. This indicates that the parameters $\gamma$ in BN are very important in evaluating a channel/neuron. In this work, we denote that $\gamma$ and $\beta$ are the approximation of standard deviation and mean in a batch. If $\gamma$ of a channel is small, the output of this channel is approximated by a constant and contains little information for the next layer. Thus, it is reasonable to prune channels with small $\gamma$ as ~\cite{liu2017slim} does.
%But from our view, we find that units with similar $\gamma$ and $\beta$ have similar outputs, and the performance of the network can still be maintained by removing some of the similar units. Therefore, based on our observation, we relax this idea and propose a new network pruning method that not only units with small $\gamma$, but also units with similar $\gamma$ can be removed. 
Thirdly, the scaling invariant property and symmetry of normalization methods need to be carefully studied since it can influence the stability of network training. 

\subsection{Contributions}
To tackle the aforementioned issues, we propose a corresponding scheme. More specifically, our contributions are three-fold: (1) we first propose a unified analysis tool for most of the existing normalization methods. Then, we make a deep analysis of normalization methods according to the proposed tools and point out the relationship between them; (2) we put forward some useful claims and conclusions. First, almost all the normalization methods can be interpreted as normalizing pre-activations or weights onto a sphere or ellipsoid; Second,  conducting optimization on a sphere with scaling symmetry removed could help to make the training of network more stable; (3) we show that training with these normalization methods will keep weights increasing, which will cause adversarial vulnerability since it will amplify the attack. 

The rest of the paper is organized as follows. In Section 2, we interpret different normalization methods in a unified analysis framework and propose some useful claims and conclusions. We conduct several experiments to verify these claims in Section 3 and conclude in Section 4. 

%For any neuron in MLP or channel in CNN, the order of batch data can be arbitrary, so what matters most is the mean and variance. For every neuron, batch normalization is equivalent to constrain the pre-activation vector on an sphere with certain center and radius. So training the network can be regarded as find appropriate spheres which can approximate the pre-activations.
%In this way, we can understand the scaling and shifting after standardizing the vector is to find a common sphere for all the data. Also, it constrain every layer's output to be bounded by two learnable parameters. This can be regarded as compactification of functions in a bounded set in the space of networks function space~\cite{petersen2018topological}. Also the scaling invariant property of weights help us do optimization on a compact set in functional space and this makes normalization is very useful in training.

%All the normalizations we mentioned are scaling invariant for weights in the network, this can help us to reduce the optimization on a compact manifold like~\cite{cho2017riemannian}~\cite{Meng2019GSGDOR}. ~\cite{petersen2018topological} have pointed out weight explosion is tightly related to the non-compactness of functional set. That's why normalization is indispensable in training of deep learning. So in this view, normalizations can be regarded as compactification in functional space.

\section{Interpretations of Normalization Methods}
\label{headings}
As mentioned in the introduction section, a variety of normalization methods have been proposed in recent years~\cite{ioffe2015BN,ulyanov2016IN,wu2018GN,salimans2016WN,huang2017CWN,miyato2018SN}. In this section, we first give a unified analysis tool by proposing a lemma to exhibit the basic structure underlying these methods. Then, we make a deep analysis of each normalization method and put forward some useful claims and conclusions.

\begin{figure*}
	\centering
	\includegraphics[width=0.8\linewidth]{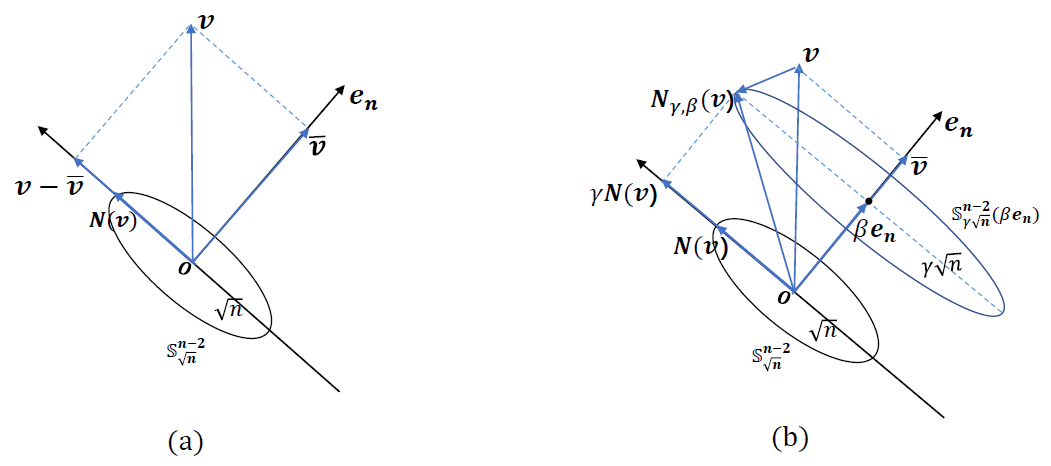}
	\caption{(a) The mean vector $\overline{v}$ of $v$ is its projection on $e_n$. $v-\overline{v}$ is orthogonal to $e_n$ and parallel to $N(v)\in \mathbb{S}^{n-2}_{\sqrt{n}}$. (b) Given $\gamma$ and $\beta$, $N_{\gamma,\beta}(v)$ maps each $v$ onto $\mathbb{S}^{n-2}_{\gamma \sqrt{n}}(\beta e_n)$, which implies that vector $N_{\gamma,\beta}(v)$ lies on a sphere with radius $\gamma \sqrt{n}$ and center $\beta e_n$.} \label{fig1}
\end{figure*}

\subsection{A Unified Analysis Tool}
Before analyzing each normalization method, we first propose a lemma to define the related notations used throughout the paper.
\begin{lemma}\label{unit_tool}
Given a vector $v=(v_1,v_2,...,v_n)^T \in \mathbb{R}^n$, \\
1) its mean vector $\overline{v}=\frac{1}{n}e_n^Tve_n=\frac{1}{n}e_ne_n^Tv$, where direction $e_n=(1,1,...,1)^T \in \mathbb{R}^n$;\\
2) the centered vector $v-\overline{v}=P_{e_n}v$, where $P_{e_n}= I_n-\frac{1}{n}e_ne_n^T$ is the projection to the hyperplane with direction $e_n$. $P_{e_n}$ has some properties, such as $P_{e_n}^2=P_{e_n}$, and $P_{e_n}v \perp e_n$;\\
3) its variance $\sigma_v^2=\frac{1}{n}(v-\overline{v})^T(v-\overline{v})=\frac{1}{n}\|v-\overline{v}\|_2^2=\frac{1}{n}v^TP_{e_n}v$, thus $\sigma_v=\frac{1}{\sqrt{n}}\|v-\overline{v}\|_2$;\\
4) the standardization
    \begin{equation}\label{standardization}
    N(v)=\frac{v-\overline{v}}{\sigma_v}=\sqrt{n}\frac{v-\overline{v}}{\|v-\overline{v}\|_2}=\frac{\sqrt{n}}{\|P_{e_n}v\|_2}P_{e_n}v, 
    \end{equation}
    and the norm is 
    \begin{equation}
        \|N(v)\|_2=\sqrt{n}
    \end{equation}
    i.e. $N(v) \in \mathbb{S}^{n-2}_{\sqrt{n}}$, where $\mathbb{S}^{n-2}_{\sqrt{n}}$ represents the $n-2$ dimension sphere with radius $\sqrt{n}$ and center 0 (omitted when center at origin); Also, $N(v)$ is scaling and moving invariant, i.e. $N(\alpha v+ te_n)=N(v), \forall \alpha,t \in \mathbb{R}$;\\
5) we have the orthogonal decomposition
\begin{equation}\label{orthogonal}
v= \gamma N(v) + \beta e_n, \\ 
\end{equation}
where \ $\gamma=\sigma_v=\frac{\|v-\overline{v}\|_2}{\sqrt{n}}, \beta=\frac{1}{n}e_n^Tv.$ The norm is \\
\begin{align}
\label{norm}
\|v\|_2^2=\|v-\overline{v}\|_2^2+\|\overline{v}\|_2^2  
&=\gamma^2\|N(v)\|_2^2+n\beta^2  \nonumber \\
&=n(\gamma^2+\beta^2).
\end{align}
If $\gamma$ and $\beta$ are fixed, $\gamma N(v) + \beta e_n \in \mathbb{S}^{n-2}_{\gamma \sqrt{n}}(\beta e_n)$. 
\end{lemma}

\begin{proof}
1) The mean of n-dimension vector $v\in \mathbb{R}^n$ is
\begin{align}
    &\frac{1}{n}(v_1+v_2+...+v_n)\\ \nonumber
    &=\frac{1}{n}(1,1,...,1)\cdot(v_1,...,v_n)^T \\ \nonumber
    &=\frac{1}{n}e_n^Tv \in \mathbb{R}
\end{align}

Note that the last equality holds as the multiplicity of scalar and vector can be commutative. Thus the mean vector\\
\begin{align}\label{mean_verctor}
    \overline{v}&=\frac{1}{n}(v_1+v_2+...+v_n)e_n\\ \nonumber
    &=\frac{1}{n}(e_n^Tv)e_n=\frac{1}{n}e_n(e_n^Tv)=\frac{1}{n}e_ne_n^Tv
\end{align}

2) By (\ref{mean_verctor}), the centered vector\\
\begin{align}
    v-\overline{v}=v-\frac{1}{n}e_ne_n^Tv=(I_n-\frac{1}{n}e_ne_n^T)v=P_{e_n}v
\end{align}
Now we prove $P_{e_n}$ is a projection matrix, i.e., $P_{e_n}^2=P_{e_n}$.\\
\begin{align}
    P_{e_n}^2\nonumber
    &=(I_n-\frac{1}{n}e_ne_n^T)(I_n-\frac{1}{n}e_ne_n^T)\\ \nonumber
    &=I_n-\frac{2}{n}e_ne_n^T+\frac{1}{n^2}e_n(e_n^Te_n)e_n^T\\ \nonumber
    &=I_n-\frac{2}{n}e_ne_n^T+\frac{1}{n}e_ne_n^T\\ \nonumber
    &=I_n-\frac{1}{n}e_ne_n^T=P_{e_n}
\end{align}
Note that the last equality holds as $e_n^Te_n=n$ by the definition of $e_n$.\\
Also we have \\
\begin{align}
    e_n^TP_{e_n}v\nonumber
    &=e_n^T(I_n-\frac{1}{n}e_ne_n^T)v\\ \nonumber
    &=(e_n^T-\frac{1}{n}(e_n^Te_n)e_n^T)v\\ \nonumber
    &=(e_n^T-e_n^T)v=0 
\end{align}
So $P_{e_n}v \perp e_n$, and $P_{e_n}$ is the projection operation onto $e_n^{\perp}$.\\
3) The variance of components in $v$ is
\begin{align}
    \sigma_v^2\nonumber
    &=\frac{1}{n}\sum_{i=1}^n(v_i-\frac{1}{n}e_n^Tv)\\ \nonumber
    &=\frac{1}{n}(v-\overline{v})^T(v-\overline{v})=\frac{1}{n}\|v-\overline{v}\|_2^2 \\ \nonumber
    &=\frac{1}{n}\|P_{e_n}v\|_2^2
\end{align}
Thus, $\sigma_v=\frac{1}{\sqrt{n}}\|v-\overline{v}\|_2=\frac{1}{n}\|P_{e_n}v\|_2$.\\
4) We only need to show that $N(v)$ is scaling and translation invariant by above results. $\forall \alpha,t \in \mathbb{R}$,
\begin{align}
    \nonumber
    &N(\alpha v+ te_n)\\ \nonumber
    &=\sqrt{n}\frac{(\alpha v+ te_n)-(\alpha \overline{v})+ te_n)}{\|(\alpha v+ te_n)-(\alpha \overline{v})+ te_n)\|_2}\\ \nonumber
    &=\sqrt{n}\frac{\alpha v-\alpha \overline{v}}{\|\alpha v-\alpha \overline{v}\|_2}=N(v)\nonumber
\end{align}
The norm of $N(v)=\sqrt{n}$ by definition, which means that $N(v)$ always lies on a sphere with radius $\sqrt{n}$ and center $0$ which is orthogonal to fixed direction $e_n$. Thus $N(v) \in \mathbb{S}^{n-2}_{\sqrt{n}}$. \\  
5) Now we prove the orthogonal decomposition of (\ref{orthogonal}).
\begin{align}
    v&=v-\overline{v}+\overline{v}\\ \nonumber
    &=\sigma_v\frac{v-\overline{v}}{\sigma_v}+\frac{1}{n}e_n^Tve_n\\ \nonumber
    &=\sigma_vN(v)+\frac{1}{n}e_n^Tve_n\\ \nonumber
    &=\gamma N(v) + \beta e_n
\end{align}
where  $\gamma=\sigma_v=\frac{\|v-\overline{v}\|_2}{\sqrt{n}}, \beta=\frac{1}{n}e_n^Tv.$\\
Also from 2) we learn that $N(v) \perp e_n$, thus we have\\
\begin{align}
    \|v\|_2^2&=\|v-\overline{v}\|_2^2+\|\overline{v}\|_2^2\\ \nonumber
    &=\gamma^2\|N(v)\|_2^2+n\beta^2\\ \nonumber
    &=n(\gamma^2+\beta^2)
\end{align}
Then for any vector we can composite it as two orthogonal direction $N(v)$ and $e_n$. If we fix the coefficients $\gamma,\beta$ by the definition of $N(v)$, $\gamma N(v) + \beta e_n \in \mathbb{S}^{n-2}_{\gamma \sqrt{n}}(\beta e_n)$.

\end{proof}

To better illustrate this lemma, a sketch map is shown in Fig.~\ref{fig1}. To sum up, this lemma gives us an overall picture of understanding normalization. Most of the normalization methods standardize data or weights. More specifically, the main idea is to decouple the length and direction and optimize them separably~\cite{kohler2019exponential}. Additionally, it can also be seen from the lemma that the standardized vectors lie on a sphere, and all directions are equilibrium on the sphere. Therefore, we can do optimization on a compact sphere instead of the whole parameter space. In the following, we will use this lemma to make a deep interpretation of some of the normalization methods.

\subsection{Batch Normalization}\label{batchnorm}
In this section, we make a deep analysis of batch normalization (BN) based on Lemma \ref{unit_tool}. For simplicity, we only analyze one layer in a multi-layer perceptron (MLP), while it is similar to analyze the convolution layer of CNN (more details can be seen in the Appendix of the supplemental material). There are also extensively works on batch normalization, for example, why it works~\cite{yang2019mean,cho2017riemannian,bjorck2018understanding,kohler2019exponential} and how to use it in concrete tasks~\cite{park2019effect}, while we focus on its geometric properties.

Before conducting the analysis, some notations are defined. $\{X^k\}_{k=1}^B, X^k \in \mathbb{R}^n, k=1,...,B$, denotes as a batch set with $B$ samples, and the input for one layer can be represented as $X=(X^1,...,X^B)\in \mathbb{R}^{n\times B}$. We assume that the layer of MLP contains $m$ hidden nodes, and thereby the weight matrix of this layer can be denoted as $W=(W_1,W_2,...,W_m)^T\in \mathbb{R}^{m\times n}$, where each row vector $W_i\in \mathbb{R}^n, i=1,...,m$. Therefore, the output $Y^k=(Y^{k}_{1},Y^{k}_{2},\dots,Y^{k}_{m})\in\mathbb{R}^m$ of this layer can be represented as $Y^k=WX^k$. Additionally, we also define an all-one vector $e_B=(1,1,...,1)^T\in \mathbb{R}^B$.

Based on Lemma \ref{unit_tool}, batch normalization for $Y^{k}_i$ can be expressed as
\begin{align}\label{BN_2}
    &BN_{\gamma,\beta}(Y^{k}_i)\\ \nonumber
    &=\gamma\sqrt{B}\frac{Y^{k}_i-\bar{Y}_i}{\|Y^{k}_i-\bar{Y}_i\|_2}+\beta e_B^T \in \mathbb{S}^{B-2}_{\gamma\sqrt{B}}(\beta e_B^T)
\end{align}
where $Y^{k}_i$ is the $i_{th}$ element of pre-activation vector $Y^{k}$, and $\bar{Y}_i$ is the mean of components for $Y^{k}_i$. Besides, $\|BN_{\gamma,\beta}(Y_i)\|_2=\sqrt{B(\gamma^2+\beta^2)}$. According to this new expression of BN, we make several discoveries. 

Firstly, since $BN_{\gamma,\beta}(Y_i)\in \mathbb{S}^{B-2}_{\gamma\sqrt{B}}(\beta e_B^T)$, training with BN can be interpreted as finding a fixed sphere, where the pre-activation vector $Y^{k}_i$ can be approximated well.

Secondly, the parameters $\gamma$ and $\beta$ are very important as they determine the sphere for pre-activation vectors. The only difference between units at training is the two parameters, and they can be treated as the characteristics of units. More concretely, units with similar $\gamma$ and $\beta$ have similar outputs and thus similar functions. Therefore, the performance of the network can still be maintained by removing some of the similar units. A similar idea has been applied to network pruning~\cite{liu2017slim} by removing some units with small $\gamma$. But in our view, this idea can be further relaxed based on our observation. Specifically, this idea can be expressed as that not only units with small $\gamma$, but also units with similar $\gamma$ and $\beta$ can be removed. We will leave this method as future work. 
%Therefore, we propose a new network pruning strategy based on BN. 
%In the experimental section, we will conduct some experiments to verify this strategy.

Thirdly, it also should be noted that the order of data in a batch does not change the result, thus the direction $N(v)$ can be any permutation of the components. There are $B!$ possible directions in a batch in total, and these directions should have the same output.

Fourthly, the batch size is a very important hyperparameter for training with BN. That too-small batch size will not make the training stable and large batch size would decrease generalization as studied in~\cite{luo2018towards}. Batch normalization for $Y_i$ can be represented as
\begin{equation}\label{BN_weight}
BN(Y_i)=\frac{W_i(X-\bar{X})}{\sqrt{\frac{1}{B}\|W_i(X-\bar{X})\|_2^2}}=\frac{W_i(X-\bar{X})}{\sqrt{W_i\Sigma_XW_i^T}}
\end{equation}
where each row of $\bar{X}$ is the mean of row in $X$, $\bar{X}=\frac{1}{B}Xe_Be_B^T$, $\Sigma_X$ is the covariance
matrix of row vectors in $X$, and $\Sigma_X=\frac{1}{B}(X-\bar{X})(X-\bar{X})^T$. We denote the rank of $\Sigma_X$ as $rank(\Sigma_X)$, and $rank(\Sigma_{X})\leq min(n, B)$. Next, we will have a further study about this by rank analysis.
\begin{itemize}
    \item Let $K_{\Sigma_X}=\{W_i\in \mathbb{R}^n|W_i\Sigma_X=0\}$ be the kernel space of $\Sigma_X$. If $W_i\in K_{\Sigma_X}$, then $BN_{\gamma,\beta}(Y_i)=\beta e_B^T$, which means that the sphere collapses to a mean vector. In this case, the generalization of network will decrease as studied in~\cite{morcos2018importance}.
    \item If batch size $B\ll n$, the dimension of kernel $dim(K_{\Sigma_X})\geq n-B$ will be large. That is to say
    $W_i$ will lie in or very close to $K_{\Sigma_X}$ with very high probability. For example, in the extreme case where $B=2$, $BN(Y_i)=(-1,1)$ or $BN(Y_i)=(1,-1)$.
    \item If batch size $B\geq n$, this setting may avoid the above situation to some extent, but in real application, we cannot have a very big batch size due to the memory constraint. Also large batch would decrease generalization as studied in~\cite{luo2018towards}. Thus, choosing an appropriate batch size is an important question. 
    %From above analysis, we conjecture that batch size close to width of the network obtain both good generalization and stable training.
\end{itemize}
  
Additionally, from Eq.~(\ref{BN_weight}), we can see that BN puts constraint on row vectors of weight matrix $W$, namely
\begin{eqnarray}\label{bn_weight}
BN_{\gamma,\beta}(Y_i)&=&\frac{\gamma W_i}{\sqrt{W_i\Sigma_XW_i^T}}(X-\bar{X})+\beta e_B^T \nonumber\\
&=&W_i'X+b',
\end{eqnarray}
where $W_i'=\frac{\gamma W_i}{\sqrt{W_i\Sigma_XW_i^T}}$, and $b'=\beta e_B^T-\frac{\gamma W_i\bar{X}}{\sqrt{W_i\Sigma_XW_i^T}}$. Therefore, BN is scaling invariant for both weights $W$ and inputs $X$, also invariant with bias with row vectors of $X$. We can see that $W_i'$ lies on a $rank(\Sigma)$ dimension ellipsoid~\cite{cho2017riemannian}~\cite{kohler2019exponential}, and the ellipsoid converges as the training converges. Thus, batch normalization constrains our effective weight $W_i'$ on an ellipsoid which is a compact set. We should notice that the ellipsoid is changing over training steps, which means the units adjust the effective weights to achieve better performance.

\subsection{Layer Normalization}
In this section, we make a deep analysis of layer normalization based on Lemma \ref{unit_tool}. Specifically, layer normalization (LN)~\cite{lei2016LN} normalizes inside a layer, whose input is $X=(X_1,X_2,...,X_n)^T\in \mathbb{R}^n$, and the weight is $W \in \mathbb{R}^{m\times n}$. Thus $Y=(Y_1,Y_2,...,Y_m) \in \mathbb{R}^m$. Let $e_m=(1,1,...,1)^T\in \ \mathbb{R}^m$, then the mean of $Y$ is $\bar{Y}=\frac{1}{m}e_m\sum_{i=1}^{m}Y_i=\frac{1}{m}e_me_m^TY$. Layer normalization for $Y$ is
%\begin{equation}\label{ln_2}
%LN_{\gamma,\beta}(Y)=\gamma\sqrt{m}\frac{Y-\bar{Y}}{\|Y-\bar{Y}\|_2}+\beta e_m
%\end{equation}
\begin{eqnarray}\label{ln_weight}
% \nonumber % Remove numbering (before each equation)
LN_{\gamma,\beta}(Y)&=&\gamma\sqrt{m}\frac{Y-\bar{Y}}{\|Y-\bar{Y}\|_2}+\beta e_m \nonumber
\\&=&\gamma\sqrt{m}\frac{(W-\bar{W})X}{\|(W-\bar{W})X\|_2}+\beta e_m
\end{eqnarray}
where $\bar{W}=\frac{1}{m}e_me_m^TW$. LN is scaling invariant for both weights $W$ and inputs $X$, and is moving invariant with row vectors of $W$. Every row of $\bar{W}$ is
\begin{align*}\label{ln_weight_mean}
\bar{W}_i=(\frac{1}{m}\sum_{k=1}^{m}W_k^1,\frac{1}{m}\sum_{k=1}^{m}W_k^2,...,\frac{1}{m}\sum_{k=1}^{m}W_k^n) \in \mathbb{R}^n 
\end{align*}
for $i=1,...,m.$ \nonumber  

Thus $\bar{W}_i$ is the mean of all the row vectors in $W$. It is easily observed from above Eq. (\ref{ln_weight}) that LN is equivalent to centering the row vectors of $W$ and then normalizing $(W-\bar{W})X$ on $\mathbb{S}^{m-2}_{\sqrt{m}}$ (more details can be found in Appendix).

\subsection{Weight-based Normalization}
In this section, we make a deep analysis of weight-based normalization methods. This kind of normalization constrains the effective weights on a sphere rather than ellipsoid as data-based normalization does. In this way, data-based normalization is much flexible to fit data and obtain good performance. It seems that both BN and LN do normalization to data, but they implicitly do the transformation of weights as characterized above. IN and GN follow the same way as LN, so we do not present the deduction here. First, weight normalization (WN)~\cite{salimans2016WN} is a direct way to normalize weight rather than inputs. It splits the weights' row vector as length and direction like BN and updates both of them in the training. WN takes the form as
\begin{eqnarray}
% \nonumber % Remove numbering (before each equation)
W_i=g\frac{V_i}{\|V_i\|_2}
\end{eqnarray}
where $g$ is the length of $W_i$ and $V_i$ is the direction of the row vector. We can see that if $\Sigma_{X}=Id$, then BN degenerates to WN. That is, if the covariant matrix is an identity or the vectors of the data matrix in every layer are independent, then BN will degenerate to WN. Of course, it is almost impossible to make all the layers independent in BN. BN performs better and also widely used in many tasks. One of the possible reasons is that it considers the relationship between different samples which shows as the direction of the vector in Fig 2.

Spectral normalization (SN)~\cite{miyato2018SN} was proposed to train the discriminator of GAN and has obtained better stability. The idea is to constrain the Lipschitz constant of the network by normalizing weight. Precisely,
\begin{eqnarray}
% \nonumber % Remove numbering (before each equation)
W_{SN}(W)=\frac{W}{\sigma(W)}
\end{eqnarray}
where $\sigma(W)=\max_{\|h\|_2=1}\|Wh\|_2$ is the spectral norm of $W$, namely the biggest singular value of $W$.

Centered weight normalization (CWN)~\cite{huang2017CWN} and weight standardization (WS)~\cite{qiao2019WS} both normalize weights' row vector by centering and scaling, 
\begin{eqnarray}
% \nonumber % Remove numbering (before each equation)
WS(W_i)&=&\frac{W_i-\bar{W}_i}{\sigma_{W_i}}=\sqrt{n}\frac{W_i-\bar{W}_i}{\|W_i-\bar{W}_i\|_2} \nonumber \\
&=&\sqrt{n}CWN(W_i)
\end{eqnarray}
where $\bar{W}_i=\frac{1}{n}W_ie_ne_n^T$ is mean of row vector $W_i$, and $e_n=(1,1,...,1)^T\in \ \mathbb{R}^n$. WS is equivalent to CWN which only normalizes by $L_2$ norm of a centered row vector. WS puts the row vector of weight on a sphere $\mathbb{S}^{n-2}_{\sqrt{n}}$, while CWN moves row of weight on unit sphere $\mathbb{S}^{n-2}_{1}$.

\begin{figure*}[h]
	\centering
	\includegraphics[width=0.85\linewidth]{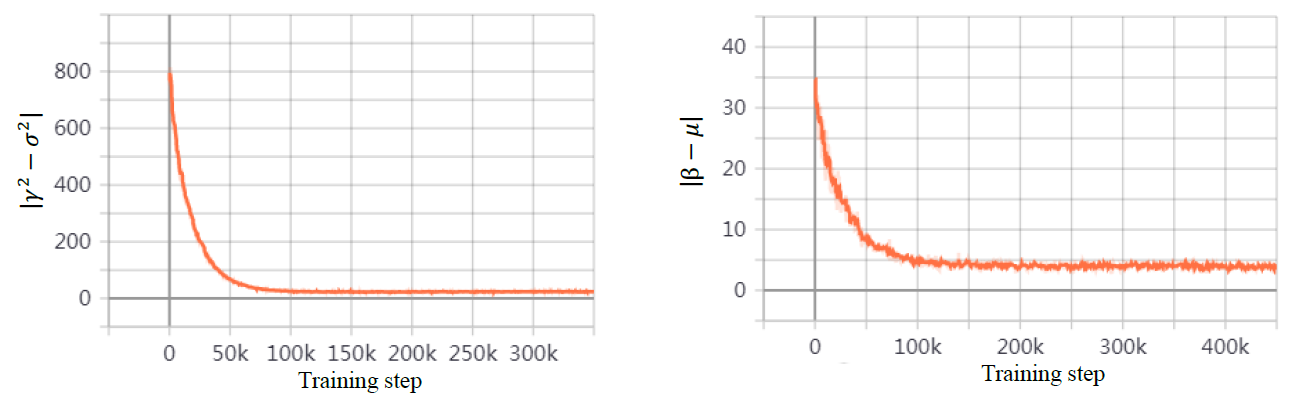}
	\caption{The change of $|\gamma^2-\sigma^2|$ and $|\beta-\mu|$ in training process.} 
	\label{fig2}
\end{figure*}

\begin{table*}[h]
    \centering
    \begin{tabular}{|c|c|c|c|c|c|c|c|c|c|}
    \hline
     Model & BN & BN+WD & GN & GN+WD & LN & LN+WD & WN & WN+WD \\
     \hline 
     Clean & 0.9687 & 0.9375 & 0.9687 & 0.9375 & 0.8711 &  0.9687 & 0.8679 & 0.9062\\
     Gaussian Noise & 0.8125 & 0.7812 & 0.9062 & 0.8437 & 0.8437 & 0.8002 & 0.7060 & 0.7187\\
     Acc-Diff1 & 0.1562 & 0.1563 & 0.0625 & 0.0938 & 0.0274 & 0.1685 & 0.1619 & 0.1875\\
     BIM-$l_\infty$ & 0.5380 & 0.6000 & 0.5793 & 0.5815 & 0.5253 & 0.5333 & 0.5031 & 0.4687\\
     Acc-Diff2 & 0.4307 & 0.3375 & 0.3894 & 0.3560 & 0.3458 & 0.4354 & 0.3648 & 0.4375\\
     \hline
    \end{tabular}
    \caption{Test accuracy of ResNet101 on CIFAR-10}
    \label{tab:test}
\end{table*}

\subsection{Symmetry of Normalization}
Most of the aforementioned normalization methods are scaling invariant, which will lead us to do optimization on the space by dividing the scaling symmetry \cite{cho2017riemannian, Meng2019GSGDOR}. To be more precise, the scaling invariant property allows us to pay more attention to the direction of weights instead of the magnitude. No matter how big or small the weights are, the effective weights don't change much. Here effective weights mean the weights multiplied by $X$, like $W_i'$ in Eq. (\ref{bn_weight}). This stables the training process and helps to obtain better performance than those without normalization.
%Usually weights are initialized at mean zero , so BN may also have the impact of regularization ~\cite{van2017l2}.

%%%%%%%%%%%%%%%%%%%%%%%%%%%%%%%%%%%%%%%%%%%%%%%%%%%%%%%%%%%%%%%%%%%%%%%%%%%%%
We need another lemma from the symmetric group view to give us useful intuitions about normalization. In classical mechanics, symmetry is always related to invariant and is identified as Nother's Theorem~\cite{arnol2013mathematical}. In our case, because the weights are scaling invariant for many normalization methods, we just study the identity related to the scaling group. It tells us that the weight is always orthogonal to its gradient, thus the norm of weights will keep increasing with stochastic gradient descent (SGD).

\begin{lemma}\label{nother-theorem}
If $f(x)$ is a scaling invariant differentiable function on $\mathbb{R}^n$, i.e. $f(\lambda x)=f(x), \forall \lambda \in \mathbb{R}$, then
\begin{eqnarray}
% \nonumber % Remove numbering (before each equation)
 x^T\nabla f(x)\equiv 0. 
\end{eqnarray}
where $\nabla f(x)$ is the gradient at $x$. This implies that $x$ is always orthogonal to its gradient, $x \bot \nabla f(x)$.
\end{lemma}
\begin{proof}
As $f(\lambda x)=f(x)$, differentiating with respect to $\lambda$ and evaluating at $\lambda=1$, we have 
\begin{eqnarray}
% \nonumber % Remove numbering (before each equation)
  0 &\equiv& \frac{d}{d\lambda}|_{\lambda=1}f(\lambda x)= \frac{df(\lambda x)}{d(\lambda x)}\frac{d(\lambda x)}{d\lambda} \nonumber \\
   &=& x^T\frac{df}{dx}(x)=x^T\nabla f(x). \nonumber
\end{eqnarray}
Thus the proof is completed.
\end{proof}

%%%%%%%%%%%%%%%%%%%%%%%%%%%%%%%%%%%%%%%%%%%%%%%%%%%%%%%%%%%%%%%%%%%%%%%%
With the above setting, we just give an analysis on a neuron $i$ of one layer for simplicity, and the analysis applies to all the weights. Let the loss function be $\mathcal{L}(W)$. Since $\mathcal{L}(W)$ is scaling invariant for every row $W_i$ of $W$, i.e. $\mathcal{L}(\Lambda W)=\mathcal{L}(W)$, and for any diagonal matrix $\Lambda = diag(\lambda_1, \lambda_2,...,\lambda_m)\in \mathbb{R}^{mm}$. By Lemma \ref{nother-theorem}, we directly have
\begin{lemma}\label{norm_increase}
If loss function $\mathcal{L}(W)$ is scaling invariant with every row $W_i$ of $W$, then
\begin{eqnarray}
% \nonumber % Remove numbering (before each equation)
 W_i\perp \nabla_{W_i} \mathcal{L}(W).
\end{eqnarray}
Thus, for the SGD process, $W_i^{k+1}=W_i^k-\eta \nabla_{W_i} \mathcal{L}(W^k)$, the norm of $W_i^{k+1}$
\begin{eqnarray}
% \nonumber % Remove numbering (before each equation)
 \|W_i^{k+1}\|_2^2 = \|W_i^{k}\|_2^2 + \eta^2 \|\nabla_{W_i} \mathcal{L}(W)\|_2^2.
\end{eqnarray}
This implies that the norm of weights will keep increasing.
\end{lemma}

All the aforementioned normalization methods, including BN, LN, IN, GN, WN, CWN, and WS, are scaling invariant. Therefore, training with these normalization methods will keep weights increasing. This was first observed in~\cite{salimans2016WN} for WN. Here we make a further step and obtain the same result for all the scaling invariant normalization methods, which is also found in~\cite{arora2018theoretical}. 

However, the increasing weight may cause adversarial vulnerability as it will amplify the attack. As mentioned in~\cite{galloway2019batch}, batch normalization is a cause of adversarial vulnerability. We conclude that the adversarial vulnerability exists in all the scaling invariant normalization methods, not limited to batch normalization, and weight decay can reduce the vulnerability to some extent~\cite{hoffer2018norm,zhang2018three}. In the experimental section, we conduct experiments to verify this point.

Though it is easy to get Lemma \ref{norm_increase} from Lemma \ref{nother-theorem}, we also prove it for batch normalization with very careful computation as~\cite{santurkar2018BN}(see details in Appendix 1.5). Also, it is obvious that we can get the term $\langle \nabla_{\mathbf{y}_j}\mathcal{L},\mathbf{\hat{y}}_j\rangle$ vanishing in~\cite{santurkar2018BN} with BN by Lemma \ref{nother-theorem}.

\section{Experiments}
In this section, we conduct a series of experiments to verify the claims of normalization methods induced by our proposed analysis tools. Specifically, we mainly focus on batch normalization (BN) method and adversarial vulnerability for scaling invariant normalization. The experiments are conducted on CIFAR-10 or CIFAR-100 dataset where images are normalized to zero mean and unit variance. Additionally, the training samples are also augmented by left-right flipping. Since ResNet~\cite{he2016deep} has been verified to achieve the state-of-the-art performance in the image classification task, we thus use the ResNet-101 as our baseline network in the experiment.

\begin{table*}[ht]
    \centering
    \begin{tabular}{|c|c|c|c|c|c|c|c|c|c|}
    \hline
     Model & BN & BN+WD & GN & GN+WD & LN & LN+WD & WN & WN+WD \\
     \hline 
     Clean & 0.6875 & 0.6563 & 0.6257 & 0.6063 & 0.6358 & 0.6279 & 0.6147 & 0.6072\\
     Gaussian Noise & 0.4596 & 0.4953 & 0.4098 & 4327 & 0.3983 & 0.4278 & 0.3476 & 0.3874\\
     Acc-Diff1 & 0.2279 & 0.1610 & 0.2159 & 0.1736 & 0.2375 & 0.2001 & 0.2671 & 0.2198\\
     BIM-$l_\infty$ & 0.1397 & 0.1558 & 0.2366 & 0.2813 & 0.2354 & 0.2531 & 0.1647 & 0.1965\\
     Acc-Diff2 & 0.5478 & 0.5005 & 0.3891 & 0.3250 & 0.4004 & 0.3748 & 0.4500 & 0.4077\\
     \hline
    \end{tabular}
    \caption{Test accuracy of ResNet-101 on CIFAR-100}
    \label{tab:test}
\end{table*}
\begin{figure*}[h]
	\centering
	\includegraphics[width=0.75\linewidth,height=0.6\linewidth]{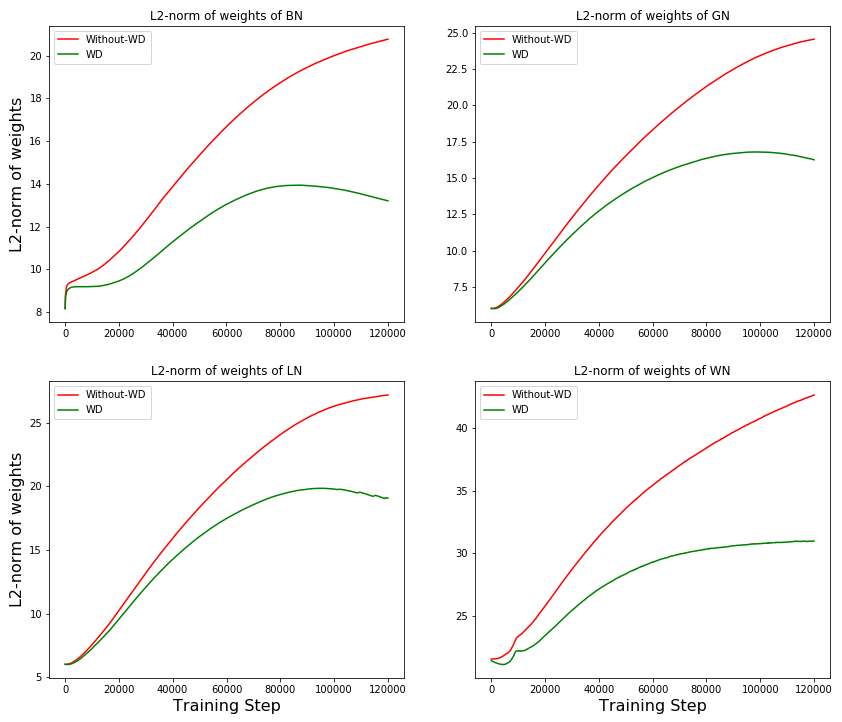}
	\caption{L2-norm of weights of different normalization methods (with or without WD) in training process.} 
	\label{weight-norm}
\end{figure*}
\subsection{Convergence of Parameters for BN}
Firstly, we conduct experiments to verify the claim that $\gamma^2$ is the approximation of variance $\sigma^2$ for units' pre-activation vector, and $\beta$ is the approximation of mean $\mu$ as shown in Section~\ref{batchnorm}. In other words, the claim indicates that the value of $|\gamma^2-\sigma^2|$ and $|\beta-\mu|$ should stabilize at a small value when the training converges.

Specifically, in this experiment, we train the ResNet-101 model on CIFAR-10 using the SGD algorithm with a learning rate of $10^{-3}$ and epoch number 200. Focusing on the first BN module in the first residual block, we show the change of $|\gamma^2-\sigma^2|$ and $|\beta-\mu|$ as iteration goes in Fig \ref{fig2}. There are 64 feature layers for BN module and each layer has its own $\gamma_{i},\sigma_{i},\beta_{i},\mu_{i}$. Since they have similar changing patterns, we randomly choose one pair of values to show in Fig \ref{fig2}. From Fig. \ref{fig2}, it can be easily seen that the parameters $\gamma$ and $\beta$ are of great importance in the training process. The convergence of $\gamma$ and $\beta$ is related to the convergence of training neural networks. The distance between $\gamma$ and $\sigma$ keeps stable at a small value after the model converges, so does the distance between $\beta$ and $\mu$. %Thus we may use this observation to guide us to prune networks and also provide a metric to determine the network's early stopping, and these would be the future works.

\subsection{Adversarial vulnerability for scaling invariant normalizations}
In this section, we show the relationship between adversarial vulnerability and scaling invariant normalization methods. As we claim in Lemma~\ref{norm_increase}, the norm of weights keeps increasing in the training process for scaling invariant normalization methods. We conduct experiments on CIFAR-10 and CIFAR-100 datasets using the ResNet-101 model with a learning rate of $10^{-3}$ and epoch number 200. 

Focusing on the first batch normalization module in the first residual block, we show the $L_2$-norm of weights of convolutional layers. The $L_2$-norm curves with BN, LN, GN and WN are shown in Fig.~\ref{weight-norm}. Additionally, we evaluate the robustness of convolutional networks with and without the weight decay strategy, which are trained under the same settings as before. Considering the simplicity and ability to degrade performance with little perceptible change to the input, we apply white-box adversarial attack-BIM and additive Gaussian noise to original clean data. The setting details are the same as \cite{galloway2019batch}. The test accuracy of the same model structure with different normalization methods are shown in Table~\ref{tab:test}, where "+WD" denotes whether weight decaying is added.

The weight norms of all the normalization methods we consider are increasing as training goes. By further adding the weight decay, the weight norm stabilizes early or even has a minor decrease. Also, from Table~\ref{tab:test} we can see that accuracy decreases with noise and BIM attacking in the same pattern. This implies that all these scaling invariant normalization methods have the same problem of adversarial vulnerability. One of the reasons for this problem is that increasing weight magnitude can amplify the noise or attack in the deep network. Additionally, we find that with weight decay the accuracy increases a bit than that without weight decay. That is because weight decay breaks the scaling invariant property so that the weight norm stabilizes early or even decreases, as shown in Fig.~\ref{weight-norm}.

\newpage
\section{Conclusion and Future work}
This paper makes a deep analysis on most of the normalization methods according to the unified mathematical analysis tool and then puts forward some useful claims and conclusions. More specifically, firstly, most normalization methods can be interpreted as normalizing pre-activations or weights onto a sphere. Secondly, a new network pruning strategy based on centering parameter and scaling parameter of BN is proposed, and that the batch size of BN should be set close to the width of the network is also claimed. We will explore these in our future works. Thirdly, optimization on module space which is scaling invariant can help stabilize the training of the network. Fourthly, we show that training with these normalization methods keeps weights increasing, which aggravates adversarial vulnerability since it will amplify the attack. A series of experiments have been conducted to verify the claims and conclusions. Much work still needs to explore this topic in the future.
%\clearpage
\subsubsection*{Acknowledgments}
The second author was supported by the China Postdoctoral Science Foundation funded project (2018M643655), the Fundamental Research Funds for the Central Universities, and the China NSFC project under contract 61906151.

\medskip
\small
\bibliographystyle{aaai}
\bibliography{6872-mybibfile}
\newpage
\input{supplemental}

%\appendix

\end{document}

%% file: supplemental.tex
% \documentclass{article}
% \usepackage{aaai20}
% \usepackage[utf8x]{inputenc} % allow utf-8 input
% \usepackage{url}            % simple URL typesetting
% \usepackage{booktabs}       % professional-quality tables
% \usepackage{amsfonts}  % blackboard math symbols
% \usepackage{amsmath}
% \usepackage{nicefrac}       % compact symbols for 1/2, etc.
% \usepackage{microtype}      % microtypography
% \usepackage{graphicx}
% \usepackage{natbib}

% \newtheorem{lemma}{Lemma}
% \newtheorem{theorem}{Theorem}
% \def\begeq{\begin{equation}}
% \def\endeq{\end{equation}}

% \title{Supplemental Materials}

% % The \author macro works with any number of authors. There are two commands
% % used to separate the names and addresses of multiple authors: \And and \AND.
% %
% % Using \And between authors leaves it to LaTeX to determine where to break the
% % lines. Using \AND forces a line break at that point. So, if LaTeX puts 3 of 4
% % authors names on the first line, and the last on the second line, try using
% % \AND instead of \And before the third author name.

% \begin{section}

% \maketitle
\section{Appendix}
\subsection{Batch Normalization}
\subsubsection{BN for convolution layer}
We will use Einstein summation convention for the tensor computation. Let ${X_{i,j}^k}$ be the input tensor, where $i,j$ is the index of feature map and $k$ is the input channel, We define $Y_{p,q}^l$  to be the output tensor of filter ${w_k^{l,m,n}}$, and then $Y_{p,q}^l=w_k^{l,m,n}X_{p+m,q+n}^k$. Its mean over the batch is $\bar{Y}_{p,q}^l=w_k^{l,m,n}\bar{X}_{p+m,q+n}^k$, and the variance over batch is\\
$\sigma^2(Y^l)$\\
$=w_k^{l,m,n}(X_{p+m,q+n}^k-\bar{X}_{p+m,q+n}^k)w_a^{l,b,c}(X_{p+b,q+c}^a-\bar{X}_{p+b,q+c}^a)$ \\
$=w_k^{l,m,n}w_a^{l,b,c}(X_{p+m,q+n}^k-\bar{X}^k)(X_{p+b,q+c}^a-\bar{X}^a).$\\
Thus, batch normalization for convolution is
$BN_{\gamma,\beta}(Y_{p,q}^l)$\\
$=\gamma\frac{w_k^{l,m,n}(X_{p+m,q+n}^k-\bar{X}_{p+m,q+n}^k)}{[w_k^{l,m,n}w_a^{l,b,c}(X_{p+m,q+n}^k-\bar{X}^k)(X_{p+b,q+c}^a-\bar{X}^a)]^{\frac{1}{2}}}+\beta$

\subsubsection{Normalization to a sphere}

Batch normalization for $Y_i^k$ is
\begin{equation}\label{BN_1}
BN(Y_i^k)=\frac{Y_i^k-\bar{Y}_i}{\sigma_{Y_i}}=\frac{Y_i^k-\bar{Y}_i}{\sqrt{\frac{1}{B}\|Y_i-\bar{Y}_i\|_2^2}}=\sqrt{B}\frac{Y_i^k-\bar{Y}_i}{\|Y_i-\bar{Y}_i\|_2}
\end{equation}
where $Y_i$ is the $i_{th}$ pre-activation vector, $\bar{Y}_i$ is the mean of components for $Y_i$. By Lemma 1, BN normalizes every row of $Y$ on a sphere of radius $\sqrt{B}$ in $\mathbb{R}^B$, and $BN(Y_i)\perp e_B$
\begin{equation}\label{BN_sphere}
BN(Y_i)\in \mathbb{S}^{B-2}_{\sqrt{B}}, and \ \|BN(Y_i)\|_2=\sqrt{B}.
\end{equation}
Also, BN just gives a vector on $\mathbb{S}^{B-2}_{\sqrt{B}}$ for direction. If we want to have powerful representation, scaling and bias terms are needed,
\begin{equation}\label{BN_3}
BN_{\gamma,\beta}(Y_i)=\gamma\sqrt{B}\frac{Y_i-\bar{Y}_i}{\|Y_i-\bar{Y}_i\|_2}+\beta e_B^T
\end{equation}
$BN_{\gamma,\beta}(Y_i)\in \mathbb{S}^{B-2}_{\gamma\sqrt{B}}(\beta e_B^T)$ approximate $Y_i$ on both direction $e_B$ and $BN(Y_i^k)$, also $\|BN_{\gamma,\beta}(Y_i)\|_2=\sqrt{B(\gamma^2+\beta^2)}$.

\subsubsection{Analysis for covariance kernel}
We write $Y$ as $WX$ to find the connection with weights.
\begin{eqnarray*}\label{BN_weight}
BN(Y_i)&=&\frac{W_i(X-\bar{X})}{\sqrt{\frac{1}{B}\|W_i(X-\bar{X})\|_2^2}} \nonumber \\
&=&\frac{W_i(X-\bar{X})}{\sqrt{\frac{1}{B}W_i(X-\bar{X})(X-\bar{X})^TW_i^T}} \nonumber \\ 
&=&\frac{W_i(X-\bar{X})}{\sqrt{W_i\Sigma_XW_i^T}}
\end{eqnarray*}
where each row of $\bar{X}$ is the mean of row in $X$, $\bar{X}=\frac{1}{B}Xe_Be_B^T$. $\Sigma_X$ is the covariance
matrix of row vectors in $X$, $\Sigma_X=\frac{1}{B}(X-\bar{X})(X-\bar{X})^T$, $rank(\Sigma)\leq min(n, B)$. If $W_i\in K_{\Sigma_X}$, then $BN_{\gamma,\beta}(Y_i)=\beta e_B^T$.  \\
For small batch size $B<<n$, the dimension of kernel $dim(K_{\Sigma_X})\geq n-B$ will be large and $W_i$ will lie in or very close to $K_{\Sigma_X}$ with very high probability. For extremal case $B=2$,
\begin{eqnarray*}\label{BN_size2}
BN(Y_i)&=&\frac{(Y_i^1,Y_i^2)-\frac{1}{2}(Y_i^1+Y_i^2)(1,1)}{\sqrt{\frac{1}{2}(\|(Y_i^1,Y_i^2)-\frac{1}{2}(Y_i^1+Y_i^2)(1,1)\|_2^2)}}\nonumber \\
&=&\frac{\frac{1}{2}(Y_i^1-Y_i^2)(1,-1)}{\frac{1}{2}|Y_i^1-Y_i^2|}=(1,-1) \ or\ (-1,1)
\end{eqnarray*}
In this case, the gradient will vanish and cannot train stably. $B\geq n$ may avoid this to some extend, but we cannot have very big batch size for the hardware constrain. Also large batch would decrease generalization as studied in [9]. Thus, choosing an appropriate batch size is an important question. We conjecture that when the wide equals to the batch size will maximum the performance. 

\subsubsection{Transformation for weights}
From the equation above, we can see that BN put constrain on row vectors of weight matrix $W$,
\begin{equation}\label{bn_weight1}
BN_{\gamma,\beta}(Y_i)=\frac{\gamma W_i}{\sqrt{W_i\Sigma_XW_i^T}}(X-\bar{X})+\beta e_B^T=W_i'X+b'
\end{equation}
where $W_i'=\frac{\gamma W_i}{\sqrt{W_i\Sigma_XW_i^T}}$, $b'=\beta e_B^T-\frac{\gamma W_i\bar{X}}{\sqrt{W_i\Sigma_XW_i^T}}$.\\
BN is scaling invariant for both weights $W$ and inputs $X$, also invariant with bias with row vectors of $X$. We can see that $W_i'$ lies on a $rank(\Sigma)$ dimension ellipsoid~\cite{cho2017riemannian}~\cite{kohler2019exponential}, and the ellipsoid converges as the training converge. Thus, batch normalization constrain our weight into a ellipsoid which is a compact set.

\subsection{Layer normalization}
Layer normalization~\cite{lei2016LN} normalize inside a layer, for a layer input $X=(X_1,X_2,...,X_n)^T\in \mathbb{R}^n$, the weight $W \in \mathbb{R}^{m\times n}$, so $Y=(Y_1,Y_2,...,Y_m) \in \mathbb{R}^m$. Let $e_m=(1,1,...,1)^T\in \ \mathbb{R}^m$, then the mean of $Y$ is $\bar{Y}=\frac{1}{m}e_m\sum_{i=1}^{m}Y_i=\frac{1}{m}e_me_m^TY$.
\subsubsection{Normalization to a sphere}
Layer normalization for $Y$ is
\begin{equation}\label{ln_1}
LN(Y)=\frac{Y-\bar{Y}}{\sigma_Y}=\sqrt{m}\frac{Y-\bar{Y}}{\|Y-\bar{Y}\|_2}
\end{equation}
By Lemma 1, LN normalizes the vector $Y$ on a sphere of radius $\sqrt{m}$ in $\mathbb{R}^m$, i.e. $LN(Y)\in \mathbb{S}^{m-2}_{\sqrt{m}}$. For powerful representation LN with scaling and bias is
\begin{equation}\label{ln_2}
LN_{\gamma,\beta}(Y)=\gamma\sqrt{m}\frac{Y-\bar{Y}}{\|Y-\bar{Y}\|_2}+\beta e_m
\end{equation}
$LN_{\gamma,\beta}$ transform the vectors on sphere $\mathbb{S}^{m-2}_{\sqrt{m}}$ to $\mathbb{S}^{m-2}_{\gamma \sqrt{m}}(\beta e_m)$, where $e_m=(1,1,...,1)^T\in \ \mathbb{R}^m$.
\subsubsection{Transformation for weights}
We can see what does LN do for the input $X$, as
\begin{eqnarray}\label{ln_weight}
% \nonumber % Remove numbering (before each equation)
LN(Y)&=&\sqrt{m}\frac{Y-\bar{Y}}{\|Y-\bar{Y}\|_2}=\sqrt{m}\frac{(I_m-\frac{1}{m}e_me_m^T)Y}{\|(I_m-\frac{1}{m}e_me_m^T)Y\|_2} \nonumber \\
  &=&\sqrt{m}\frac{(I_m-\frac{1}{m}e_me_m^T)WX}{\|(I_m-\frac{1}{m}e_me_m^T)WX\|_2} \nonumber \\
  &=&\sqrt{m}\frac{(W-\bar{W})X}{\|(W-\bar{W})X\|_2}
\end{eqnarray}
where $\bar{W}=\frac{1}{m}e_me_m^TW$. LN is scaling invariant for both weights $W$ and inputs $X$, also invariant with bias with row vectors of $W$. Every row of $\bar{W}$ is
\begin{equation}\label{ln_weight_mean}
\bar{W}_i=(\frac{1}{m}\sum_{k=1}^{m}W_k^1,\frac{1}{m}\sum_{k=1}^{m}W_k^2,...,\frac{1}{m}\sum_{k=1}^{m}W_k^n) \in \mathbb{R}^n 
\end{equation}
for i=1,...,m.So $\bar{W}_i$ is the mean of all the row vectors in $W$. It is clear from above equation that LN is equivalent to center the row vectors of $W$ and then normalize $(W-\bar{W})X$ on $\mathbb{S}^{m-2}_{\sqrt{m}}$.

\subsection{Proof of Lemma 2}
As $f(\lambda x)=f(x)$, differentiate with respect to $\lambda$ and evaluate at $\lambda=1$, then we have 
\begin{eqnarray}
% \nonumber % Remove numbering (before each equation)
  0 &\equiv& \frac{d}{d\lambda}|_{\lambda=1}f(\lambda x)= \frac{df(\lambda x)}{d(\lambda x)}\frac{d(\lambda x)}{d\lambda} \nonumber \\
   &=& x^T\frac{df}{dx}(x)=x^T\nabla f(x)
\end{eqnarray}

\subsection{Proof of $\langle \nabla_{\mathbf{y}_j}\mathcal{L},\mathbf{\hat{y}}_j\rangle$=0 in~\cite{santurkar2018BN} by direct computation}
We use the notation as~\cite{santurkar2018BN}, where $\mathbf{y}_j=W_j^i\mathbf{x}_i$, $\hat{\mathbf{y}}_j=\frac{\mathbf{y}_j-\bar{\mathbf{y}}_j}{\sigma_j}$, $\mathbf{z}_j=\gamma \hat{\mathbf{y}}_j+\beta$, from~\cite{santurkar2018BN} and the chain rule we know that
\begin{align}
% \nonumber % Remove numbering (before each equation)
    \nonumber
  &\frac{\partial{\hat{\mathbf{y}}_j}}{\partial {\mathbf{y}_j}} \\ \nonumber
  &= \frac{1}{\sigma_j}(I_m-\frac{1}{m}\mathbf{e}_m^T\mathbf{e}_m-\frac{1}{m}\hat{\mathbf{y}}_j^T\hat{\mathbf{y}}_j)  \\
   &= \frac{1}{\sigma_j}(I_m-\frac{1}{m}\mathbf{e}_m^T\mathbf{e}_m)(I_m-\frac{1}{m}\hat{\mathbf{y}}_j^T\hat{\mathbf{y}}_j)\ \  as \ \mathbf{e}_m\bot\hat{\mathbf{y}}_j
\end{align} 
\begin{eqnarray*}
  \frac{\partial{\hat{\mathbf{y}}_j}}{\partial{\mathbf{w}_j^l}}=\frac{\partial{\hat{\mathbf{y}}_j}}{\partial{\mathbf{y}_j^m}}\mathbf{x}_l^m
\end{eqnarray*}
\begin{equation}\label{grady}
  \frac{\partial\mathcal{L}}{\partial {\mathbf{y}_j}}=\frac{\gamma}{\sigma_j}\frac{\partial{\mathcal{L}}}{\partial{\mathbf{z}_j}}(I_m-\frac{1}{m}\mathbf{e}_m^T\mathbf{e}_m)(I_m-\frac{1}{m}\hat{\mathbf{y}}_j^T\hat{\mathbf{y}}_j)
\end{equation}
\begin{equation}\label{gradw}
  \frac{\partial\mathcal{L}}{\partial {\mathbf{w}_j}}=\frac{\gamma}{\sigma_j}\frac{\partial{\mathcal{L}}}{\partial{\mathbf{z}_j}}(I_m-\frac{1}{m}\mathbf{e}_m^T\mathbf{e}_m)(I_m-\frac{1}{m}\hat{\mathbf{y}}_j^T\hat{\mathbf{y}}_j)\mathbf{X}^T
\end{equation}
\begin{align}
    \nonumber
    &{\langle\frac{\partial\mathcal{L}}{\partial {\mathbf{y}_j}},\hat{\mathbf{y}}_j}\rangle \\ \nonumber
    &= \frac{\gamma}{\sigma_j}\frac{\partial{\mathcal{L}}}{\partial{\mathbf{z}_j}}(I_m-\frac{1}{m}\mathbf{e}_m^T\mathbf{e}_m)(I_m-\frac{1}{m}\hat{\mathbf{y}}_j^T\hat{\mathbf{y}}_j)\hat{\mathbf{y}}_j^T \nonumber \\
   &= \frac{\gamma}{\sigma_j}\frac{\partial{\mathcal{L}}}{\partial{\mathbf{z}_j}}(I_m-\frac{1}{m}\mathbf{e}_m^T\mathbf{e}_m)(\hat{\mathbf{y}}_j^T-\hat{\mathbf{y}}_j^T(\frac{1}{m}\hat{\mathbf{y}}_j\hat{\mathbf{y}}_j^T)) \nonumber \\
   &= \frac{\gamma}{\sigma_j}\frac{\partial{\mathcal{L}}}{\partial{\mathbf{z}_j}}(I_m-\frac{1}{m}\mathbf{e}_m^T\mathbf{e}_m)(\hat{\mathbf{y}}_j^T-\hat{\mathbf{y}}_j^T)=0
\end{align}

\subsection{Proof of Lemma 3 for batch normalization by direct computation}

\begin{align}
% \nonumber % Remove numbering (before each equation)
  \nonumber
  &\langle\nabla_{\mathbf{w}_j}\mathcal{L},\mathbf{w}_j\rangle \\ \nonumber
  &= \frac\gamma{\sigma_j}\frac{\partial\mathcal{L}}{\partial z_j}({I_m}-\frac1m\mathbf{e}_m^T\mathbf{e}_m))({I_m}-\frac1m\mathbf{\hat{y}}_j^T\mathbf{\hat{y}}_j)X^T\mathbf{w}_j^T\nonumber \\
   &= \frac\gamma{\sigma_j}\frac{\partial\mathcal{L}}{\partial z_j}[\mathbf{w}_jX({I_m}-\frac1m\mathbf{\hat{y}}_j^T\mathbf{\hat{y}}_j)({I_m}-\frac1m\mathbf{e}_m^T\mathbf{e}_m))]^T \nonumber\\
   &= \frac\gamma{\sigma_j}\frac{\partial\mathcal{L}}{\partial z_j}[\mathbf{y}_j({I_m}-\frac1m\mathbf{\hat{y}}_j^T\mathbf{\hat{y}}_j)({I_m}-\frac1m\mathbf{e}_m^T\mathbf{e}_m))]^T \nonumber \\
   &= \frac\gamma{\sigma_j}\frac{\partial \mathcal{L}}{\partial z_j}[(\mathbf{y}_j-\frac1m\mathbf{y}_j\mathbf{\hat{y}}_j^T\mathbf{\hat{y}}_j)({I_m}-\frac1m\mathbf{e}_m^T\mathbf{e}_m))]^T \nonumber \\
   &= \frac\gamma{\sigma_j}\frac{\partial\mathcal{L}}{\partial z_j}[(\mathbf{y}_j-\frac1m(\sigma_j\mathbf{\hat{y}}_j+\bar{y}_j)\mathbf{\hat{y}}_j^T\mathbf{\hat{y}}_j)({I_m}-\frac1m\mathbf{e}_m^T\mathbf{e}_m))]^T \nonumber \\
   &= \frac\gamma{\sigma_j}\frac{\partial \mathcal{L}}{\partial z_j}[(\mathbf{y}_j-\sigma_j\mathbf{\hat{y}}_j)({I_m}-\frac1m\mathbf{e}_m^T\mathbf{e}_m)]^T \nonumber \\
   &= \frac\gamma{\sigma_j}\frac{\partial \mathcal{L}}{\partial z_j}[\mathbf{\bar{y}}_j({I_m}-\frac1m\mathbf{e}_m^T\mathbf{e}_m)]^T=0 
\end{align}

\medskip

% \small
% \bibliographystyle{plain}
% \bibliography{mybibfile}

% \end{section} 